%% file: neurips_2019.tex
\documentclass{article}

\usepackage[preprint]{neurips_2019}

\usepackage[utf8]{inputenc} 
\usepackage[T1]{fontenc}    
\usepackage{hyperref}       
\usepackage{url}            
\usepackage{booktabs}       
\usepackage{amsfonts}       
\usepackage{nicefrac}       
\usepackage{microtype}      
\usepackage{xspace}
\usepackage{amsmath}
\usepackage{dsfont}

\input{Commands}

\usepackage[colorinlistoftodos, shadow,color=blue!30!white, 
disable
]{todonotes}
\allowdisplaybreaks

\title{Exploration-Enhanced Politex}

\author{%
Yasin Abbasi-Yadkori \And Nevena Lazi\'c \And Csaba Szepesv\'ari \And Gell\'ert Weisz 
}

\begin{document}

\maketitle

\begin{abstract}
We study algorithms for average-cost reinforcement learning problems with value function approximation. 
Our starting point is the recently proposed \alg algorithm, a version of policy iteration where the policy produced in each iteration is near-optimal in hindsight for the sum of all past value function estimates. \alg has sublinear regret guarantees in uniformly-mixing MDPs when the value estimation error can be controlled, which can be satisfied if all policies sufficiently explore the environment. Unfortunately, this assumption is often unrealistic.
Motivated by the rapid growth of interest in developing policies that learn to explore their environment in the lack of rewards (also known as no-reward learning),
we replace the previous assumption that all policies explore the environment with that 
a single, sufficiently exploring policy is available beforehand.
The main contribution of the paper is the modification of \alg to incorporate such an exploration policy in a way that  allows us to obtain a regret guarantee similar to the previous one but without requiring that all policies explore environment.
In addition to the novel theoretical guarantees, we demonstrate the benefits of our scheme on environments which are difficult to explore using simple schemes like dithering. While the solution we obtain may not achieve the best possible regret, it is the first result that shows how to control the regret in the presence of function approximation errors on problems where exploration is nontrivial.
Our approach can also be seen as a way of 
reducing the problem of minimizing the regret to learning a good exploration policy.
We believe that modular approaches like ours can be highly beneficial in tackling harder control problems.
\end{abstract}


\input{introduction}

\input{analysis}

\input{new_offpolicy}

\input{experiments}

\input{discussion}


\bibliography{biblio}
\newpage
\appendix
\input{appendix}

\end{document}

%% file: Commands.tex
\renewcommand{\epsilon}{\varepsilon}

\usepackage{graphicx}

\newcommand{\alg}{\textsc{Politex}\xspace}
\newcommand{\eealg}{\textsc{EE-Politex}\xspace}

\newcommand{\todoy}[2][]{\todo[size=\scriptsize,color=red!20!white,#1]{Yasin: #2}}
\newcommand{\todoc}[2][]{\todo[size=\scriptsize,color=blue!20!white,#1]{Csaba: #2}}
\newcommand{\todonev}[2][]{\todo[size=\scriptsize,color=green!20!white,#1]{Nevena: #2}}

\usepackage{amsmath,abbrevs}
\usepackage{amssymb}

\usepackage{amsthm}

\newcommand{\Ot}{\tilde{\mathcal{O}}}

\usepackage[capitalize]{cleveref}

\newtheorem{theorem}{Theorem}[section]
\newtheorem{lemma}[theorem]{Lemma}

 \Crefname{assumption}{Assumption}{Assumptions}

\usepackage{algorithm}
\usepackage{algorithmic}

\usepackage{color}
\usepackage{graphicx}
\usepackage{epstopdf}
\usepackage{algorithm,algorithmic}

\usepackage{nicefrac}

\newif\ifcomm
\commfalse 

\newif\iflong
\longtrue

\newcounter{assumption}
\renewcommand{\theassumption}{A\arabic{assumption}}

\newenvironment{assumption}[1][]{\begin{trivlist}\item[] \refstepcounter{assumption}%
 {\bf Assumption\ \theassumption\ {\em (#1)} } }{
 \ifvmode\smallskip\fi\end{trivlist}}

\newcommand{\norm}[1]{\Vert#1\Vert}

\newcommand{\abs}[1]{\left\vert#1\right\vert}

\newcommand{\Real}{\mathbb R}                        
\newcommand{\R}{{\mathbb{R}}}                        

\newcommand{\E}{{\mathbf E}}                         

\newcommand{\one}[1]{\mathbf 1 \left\{#1\right\}}           

\newcommand{\eps}{\varepsilon}                       

\newcommand{\argmin}{\mathop{\rm argmin}}

\newcommand{\diag}{\mathop{\rm diag}}

\newcommand{\beq}{\begin{equation}}
\newcommand{\eeq}{\end{equation}}

\ifcomm
   \newcommand\comm[1]{\textcolor{blue}{ #1}}
   \newcommand{\mtodo}[2]{\todo{{\bf #1}: #2}} 
   \def\here#1{{\bf $\langle\langle$#1$\rangle\rangle$}}

\else
   \newcommand\comm[1]{}
   \newcommand{\mtodo}[2]{}
   \def\here#1{}
\fi

\newcommand{\cD}{{\cal D}}

\newcommand{\cZ}{{\cal Z}}
\newcommand{\cX}{{\cal X}}

\newcommand{\cA}{{\cal A}}

\newcommand{\bone}{\mathbf{1}}

\def\be{\begin{equation}}
\def\ee{\end{equation}}

\def\1{\mathbf{1}}

\newname\controllermixture{{\rm \textsc{MDP-policy-mixture}}}
\newname\stableset{{\rm \textsc{independent-set}}}

\newcommand{\ip}[1]{\langle #1 \rangle}

\newcommand{\ind}[1]{\mathbb{I}\left\{#1\right\}}
\newcommand{\PR}{\overline{\mathfrak{R}}}

\newcommand{\Regret}{\mathfrak{R}}

%% file: introduction.tex
\section{Introduction}

The focus of our work are no-regret model-free algorithms for average-cost reinforcement learning (RL) problems with value function approximation.  
Most of the existing regret results 
for model-free methods apply either to settings with no function approximation \citep{StrLiWe06,AzOsMu17,JinAlBuJo18}, or to systems with linear action-value functions and special structure, e.g. \citep{mflq2019}. 
 One exception is the recently proposed \alg algorithm by \citet{politex}. \alg is a variant of policy iteration, where the policy in each phase is selected to be near-optimal in hindsight for the sum of all past action-value function estimates.
 This is in contrast to standard policy iteration, where each policy is typically greedy with respect to the most recent action-value estimate.
In uniformly-mixing Markov decision processes (MDPs) when the action-value error decreases to some level $\eps_0$ at a parametric rate, the regret of \alg scales as
$\Ot( T^{3/4} + \varepsilon_0 T)$.

The \alg regret bound requires the error of a value function estimated from $m$ transitions to scale as $\Ot(\varepsilon_0 + \Ot(\sqrt{1/m})$. Such an error bound was shown by the authors to hold 
for the least-squares policy evaluation (LSPE) method of  \citet{bertsekas1996temporal} 
whenever all policies produced by \alg sufficiently explore the feature space (the constant hidden depends on the dimension of the feature space among other problem dependent constants). 
However, this assumption concerning uniform exploration is often unrealistic. 
More generally, controlling the value function estimation error in some fixed norm (as required by the general theory of \alg) may be difficult when policies have a strong control over which parts of the state space they visit.
In this work, we instead propose a deliberate exploration scheme that ensures sufficient coverage of the state/feature space. We also propose a value function estimation approach when working with linear value function approximation that allows us to obtain performance guarantees similar to the previously mentioned guaranteed, but under milder assumptions. 

In particular, in this work we propose to address this problem by separating the problem of coming up with a policy that explores well its environment (a version of pure exploration) 
from the problem of learning to maximize a specific reward function. 
In fact, the problem of pure exploration has been the subject of a large number of previous works (e.g., 
\citet{schmidhuber1991adaptive,thrun1992active,pathakICLR19largescale,hazan2019provably} and references therein).
The new assumption that we work with only requires a single exploratory policy to be available beforehand, and interleaves the target policy with short segments of exploration and our main result will quantify how the choice of the exploration policy impacts the regret. We view our approach as a practical remedy that allows algorithm designers to focus on one aspect of the full RL problem (either learning good value functions, or learning a good exploration policy).

We propose to estimate the value of states generated by the exploration policy using returns computed from rollouts generated by the target policy -- a hybrid, in-between case that has both off- and on-policy aspects. 
We provide a Monte-Carlo value estimation algorithm that can learn from such data, and analyze its error bound in the linear function approximation case.  Learning from exploratory data allows us to obtain a more meaningful bound on the estimation error than the available results for temporal difference (TD)-based methods~\citep{Sutton-1988}, as we cover states which the target policy might not visit otherwise. Under a uniform mixing assumption, the described estimation procedure can be accomplished using a single trajectory, and the exploration scheme can be thought of as performing soft resets within the trajectory. \todoc{I think the intro is quite long, so I'd remove the next paragraph. Or shorten it to one sentence or something that was not mentioned beforehand.}
\todoy{commented out}

We complement the novel algorithmic and theoretical contributions with synthetic experiments that are chosen  to demonstrate that explicit exploration is indeed beneficial in hard-to-explore MDPs.
While our analysis only holds for linear value functions, we also experiment with neural networks used for value function approximation and demonstrate that the proposed exploration scheme improves the performance of \alg on cartpole swing-up, a problem which is known  to be hard for algorithms which only explore using simple strategies like dithering.

\section{Background}
\subsection{Problem definition}

For an integer $k$, let $[k]=\{1,2,\dots,k\}$.  
We consider MDPs specified by $S$ states,  $A$ actions, cost function $c:[S]\times [A]\rightarrow \Real$, and transition matrix $P\in\Real^{(SA)\times S}$. The assumption that the number of states is finite is non-essential, while relaxing the assumption that the number of actions is finite is less trivial.
Recall that a policy $\pi$ is a mapping from states to distributions over actions and
following policy $\pi$ means that in time step $t\in \{1,2,\dots\}$, after observing state $x$ 
an action $a$ is sampled from distribution $\pi(\cdot|x)$. 
Denoting by $\{(x_t^*,a_t^*)\}_{t=1,2,\dots}$ the sequence of state-actions generated under a baseline policy $\pi^*$,
the regret of a reinforcement learning algorithm that gives rise to the state-action sequence $\{(x_t,a_t)\}_{t=1,2,\dots}$, with respect to policy $\pi^*$, is defined as
\begin{equation}
\Regret_T =  \sum_{t=1}^T c(x_t,a_t) -  \sum_{t=1}^T c(x_t^*, a_t^*) \;.
\end{equation}
Our goal is to design a learning algorithm that guarantees a small regret with high probability with respect to a large class of baseline policies.

Throughout the paper, we make the following assumption, which ensures that the quantities we define next are well-defined  \citep{yu2009convergence}.
\begin{assumption}[Single recurrent class]
\label{assumption:recurrence}
The states of the MDP under any policy form a single recurrent class.
\end{assumption}
MDPs satisfying this condition are also known as \emph{unichain} MDPs
\citep[Section 8.3.1][]{Puterman94}.
Under Assumption~\ref{assumption:recurrence}, 
the states observed under policy $\pi$ form a Markov chain that has a unique stationary distribution, denoted by $\mu_\pi$. We also let $\nu_\pi(x, a) = \mu_\pi(x) \pi(a|x)$, which is the unique stationary distribution of the Markov chain of state-action pairs under $\pi$. We use $P_\pi$ to denote the transition matrix under policy $\pi$. Define the $SA\times SA$ transition matrix $H_\pi$ by $(H_\pi)_{(x,a),(x',a')} = P(x'|x, a)\pi(a'|x')$.
For convenience, we will interchangeably write $H_\pi(x',a'|x,a)$ to denote this value. We will also do this for other transition matrices/kernels.
Finally, let $\lambda_\pi$ be the average cost of the policy $\pi$ and $D_\pi = \diag(\nu_\pi)$. We use subscript $e$ to denote the quantities related to the exploration policy. So $\mu_e$ is the stationary distribution of the exploration policy $\pi_e$.
Part of our analysis will require the following mixing assumption on all policies.
\begin{assumption}[Uniformly fast mixing]
\label{assumption:mixing}
Any policy $\pi$ has a unique stationary distribution $\mu_\pi$. Furthermore, there exists a constant $\kappa >0$ such that for any distribution $\mu'$, \todoc{Is $\kappa$ the same for all policies?}
\[
\sup_\pi \|(\mu_\pi - \mu')^\top P_\pi \|_{1} \leq \exp(-\nicefrac{1}{\kappa}) \|\mu_\pi - \mu' \|_{1} \,.
\]
\end{assumption}

\subsection{Action-value estimation}
For a fixed policy, we define $Q_\pi(x,a)$, the action-value of $\pi$ in $(x,a)$, as the expected total differential cost incurred when the process is started from state $x$, the first action is $a$ and in the rest of the time steps policy $\pi$ is followed. Here, the differential cost in a time step is the difference between the immediate cost for that time step and the average expected cost $\lambda_\pi$ under $\pi$.

Under our assumption, up to addition of a scalar multiple of $\bone = (1,1,\dots)^\top$,  $Q_\pi$ (viewed as a vector)
is the unique solution to the Bellman equation $Q_\pi = c - \lambda_\pi \bone + H_\pi Q_\pi$. 
We will also use that given a function $Q: [S]\times[A] \to \R$, 
any \emph{Boltzmann policy}  $\pi(a|x) \propto \exp(-\eta Q(x, a))$ is invariant to shifting $Q$ by a constant.

One common approach to dealing with large state-action spaces is using function approximators to approximate $Q_\pi(x, a)$. While our algorithm works with general function approximators, we will provide an estimation approach and analysis for linear approximators, where the approximate solution  lies in the subspace spanned by a feature matrix $\Psi\in\Real^{SA\times d}$. 
With linear function approximation, value estimation algorithms solve the projected Bellman equation defined by
$\Psi w = \Pi_{\pi_e} (c - \lambda_\pi + H_\pi \Psi w)$. 
Here, $\Pi_{\pi_e} = \Psi (\Psi^\top D_{\nu_e}\Psi)^{-1}\Psi^\top D_{\nu_e}$ is the projection matrix with respect to the stationary distribution of a policy ${\pi_e}$. In \emph{on-policy} estimation, $\pi_e = \pi$, while in \emph{off-policy} estimation $\pi_e$ is a behavior policy used to collect data. 
Let $w^{\text{TD}}$ be the solution of the above equation, also known as the TD solution. Let $\widehat w$ be the on-policy estimate computed using $m$ data points. For on-policy algorithms, error bounds of the following form are available:
$\norm{\Psi \widehat w - \Psi w^{\text{TD}}}_{\nu_\pi} \le O\left(\frac{1}{\sqrt{m}}\right)$.
For convergence results of this type, see ~\citep{Tsitsiklis-VanRoy-1997,Tsitsiklis-VanRoy-1999,antos2008learning,Sutton-Szepesvari-Maei-2009,Maei-Szepesvari-Bhatnagar-Sutton-2010,lazaric2012finite,Geist-Scherrer-2014,farahmand2016regularized,liu2012regularized,liu2015finite, bertsekas1996temporal, yu2009convergence}.

Unfortunately, one issue with such weighted-norm bounds is that the error can be very large for state-action pairs which are rarely visited by the policy.  
To overcome this issue, we aim to bound the error with respect to a weighted norm where the weights span all directions of the feature space. We will assume access an exploration policy $\pi_e$ that excites ``all directions'' in the feature-space:
\begin{assumption}[Uniformly excited features]
\label{assumption:exploration-all}
There exists a positive real $\sigma$ such that for the exploration policy $\pi_e$, $\lambda_{\min}(\Psi^\top D_{{\nu_{\pi_e}}} \Psi) \ge \sigma$.
\end{assumption}
Given the exploration policy, our goal will be to bound the error in the norm weighted by its stationary distribution $\nu_{\pi_e}$.
We will describe an algorithm for which we can obtain such a bound, which estimates the value function from on-policy trajectories whose initial states are sampled from $\nu_{\pi_e}$. 
We are not aware of any finite-time error bounds for off-policy methods in the average-cost setting. 

\subsection{Policy iteration and \alg}

Policy iteration algorithms (see e.g. \citet{bertsekas2011approximate}) alternate between executing a policy and estimating its action-value function, and policy improvement based on the estimate. In \alg, the policy produced at each phase $i$ is a Boltzmann distribution over the sum of all previous action-value estimates, $\pi_{i}(a|x)\propto \exp \left(-\eta \sum_{k=0}^{i-1} \widehat Q_k (x, a) \right)$. Here, $\widehat Q_k$ is the action-value estimate at phase $k$. Under uniform mixing and for action-value errors that behave as $O(1/\sqrt{T}+\epsilon_0)$, the regret of \alg scales as $O(T^{3/4}+T\epsilon_0)$. The authors show that the error condition holds for the LSPE algorithm when all policies (not only $\pi_e$) satisfy the feature excitation. 
Since this is difficult to guarantee in practice, in this work we aim to bound all errors in the same norm, weighted by the stationary state-action distribution of an exploratory policy $\pi_e$. Thus, we only require the feature excitation condition to hold for a single known policy. 

%% file: analysis.tex
\begin{algorithm}[!t]
{\textsc{Politex}} (initial state $x_0$, exploration policy $\pi_e$, length $T$)
\begin{algorithmic}
\STATE $n=m=T^{2/5}$, $s'=\log T$, $s=T^{1/5}-s'$
\STATE  $\widehat Q_0(x, a)=0 \,\; \forall x, a$
\FOR{$i:=1,2,\dots, n$} 
\STATE Set $\pi_{i}(a|x)\propto \exp \left(-\eta \sum_{k=0}^{i-1} \widehat Q_k (x, a) \right)$
\STATE $\cZ_i =  \textsc{CollectData}(\pi_i, \pi_e, m, s, s')$
\STATE Compute $\widehat Q_{i}$ from $\cZ_i$
\ENDFOR
\STATE
\end{algorithmic}
\textsc{CollectData} ($\pi$, $\pi_e$,  $m$, $s$, $s'$):
\begin{algorithmic}
\FOR{$j:=1,2,\dots, m$} 
\STATE Execute exploration policy $\pi_e$ for $s'$ time steps and observe state $x_j$
\STATE Take action $a_j$, sampled uniformly at random, and observe state $x_j'$
\STATE Execute target policy $\pi$ from state $x_j'$ for $s$ time steps and obtain trajectory 
\[
R_j=\{x'_j, a^{(j)}_0, c_0^{(j)}, x_1^{(j)}, a_1^{(j)}, c_1^{(j)} \dots,  x_{s-1}^{(j)}, a_{s-1}^{(j)}, c_{s-1}^{(j)}\}
\]
\ENDFOR
\STATE Return data  $\cZ =  \{(x_1,a_1,R_1), \dots, (x_m,a_m,R_m)\}$

\end{algorithmic}
 \caption{\eealg}
\label{alg:eepolitex}
\end{algorithm}

\section{Exploration-enhanced \alg}

The proposed algorithm, which we will refer to as \eealg (exploration-enhanced \alg) is shown in Figure~\ref{alg:eepolitex}.  Compared to \alg, the main difference is that we assume  access to a fast-mixing exploration policy which spans the state space, and we run that policy in short segments at a fixed schedule. Intuitively, the exploration policy serves the purpose of performing soft resets of the environment to a random state within a single trajectory. 
The action-value function of each policy is then estimated using on-policy trajectories, whose initial states are chosen approximately i.i.d. from the stationary state distribution of the exploration policy.  We assume access to an estimation black-box which can learn from such data. In the next section, we show a concrete least-squares Monte Carlo (LSMC) algorithm with an error bound of $\eps(\delta, m) = \eps_0 + C \sqrt{\log(1/\delta) / m}$ when run on $m$ data tuples, where $\eps_0$ is the worst-case approximation error. Our value-estimation requirement is weaker than that in \citet{politex}, since we provide exploratory data to the estimation blackbox, rather than requiring each policy produced by \alg to sufficiently explore. However, as we show next, the exploration segments come at a price of a slightly worse regret of $O(T^{4/5})$ compared to $O(T^{3/4})$.

\subsection{Regret of \eealg}
Consider any learning agent that produces the state-action sequence $\{(x_t,a_t)\}_{t=1,2,\dots}$ while interacting with an MDP. For a fixed time step $t$, let $\pi_{(t)}$ denote the policy that is used to generate $a_t$. 
 Let $T_e$ be the rounds that the exploration policy is played, i.e. $T_e = \{t: \pi_{(t)} = \pi_e\}$, and let 
 $E_T = \sum_{t \in T_e} \lambda_{\pi_e} - \lambda_{\pi^*}$ be the pseudo-regret in those rounds. 
 Similarly to \citet{politex}, we decompose regret into pseudo-regret and noise terms:
\begin{align}
\label{eq:regret-terms}
\Regret_T &= E_T + \PR_T + V_T + W_T \,, \\
\PR_T &=  \sum_{t\in [T] \setminus T_e} (\lambda_{\pi_{(t)}} - \lambda_{\pi^*}),\; V_T = \sum_{t \in [T] \setminus T_e}  (c(x_t, a_t) - \lambda_{\pi_{(t)}}),\;
W_T =  \sum_{t \in [T] \setminus T_e}  (\lambda_{\pi^*} -  c(x_t^{*}, a_t^{*}))\,. \notag
\end{align}

We first bound the pseudo-regret by a direct application of  Theorem~4.1 of \citet{politex} to rounds $[T] \setminus T_e$.
\begin{theorem}
\label{eq:poliprgen}
Fix $0<\delta<1$. Let $\epsilon(\delta,m)>0$ and $Q_{\max}>0$ and $b\in \R$ be such that
 for any $i\in [n]$, 
with probability $1-\delta$, 
\begin{align}
\label{eq:ver}
\norm{ Q_{\pi_{i}} - \widehat Q_i }_{\nu^*}, \,
\norm{ Q_{\pi_{i}} - \widehat Q_i }_{\mu^*\otimes\pi_i} \le \epsilon(\delta,m)
\end{align}
and $\widehat Q_i(x,a) \in [b,b+Q_{\max}]$ for any $(x,a)\in \cX \times \cA$.
Letting  $\eta=\sqrt{8 \log(A)/n}/Q_{\max}$,
with probability $1-\delta$, the regret of \alg relative to the reference policy $\pi^*$ satisfies
\begin{align*}
\PR_T & \le  2T \,\epsilon(\delta/(2n),m)  
+ n^{1/2} m s Q_{\max}  S_\delta(A,\mu^*) \,,
\end{align*}
where
$S_\delta(A,\mu^*) =  \sqrt{ \frac{\log(A)}{2}} +
\left \langle \mu^*, \sqrt{ \frac{ \log (2/\delta) + \log(1 / \mu^*)}{2}} \right \rangle 
$
\,.
\end{theorem}

Next, we bound the noise terms $W_T$ and $V_T$ using a modified version of Lemma~4.4 of \citet{politex} that accounts for additional policy switches due to exploration:
\begin{lemma} 
\label{lem:bound-VW}
Let \cref{assumption:mixing} hold. If the algorithm has a total of $2nm$ policy restarts, and each policy is executed for at most $s$ timesteps, then with probability at least $1-\delta$, we have that
\begin{align*}
 |W_T| \leq \kappa + 4\kappa\sqrt{2T\log\left(2/{\delta}\right)} \;\;\; {\rm and} \;\;\;
|V_T| \leq nm\kappa + 4 n \kappa\sqrt{4 ms \log\left(2ms/{\delta}\right)}.
\end{align*}
\end{lemma}
The proof is given in Appendix~\ref{sec:vwproof}.

Assume that the action-value error bound is of the form $\eps(\delta, m) = O(\eps_0 + C \sqrt{\log(1/\delta) / m})$, where $\eps_0$ is the irreducible approximation error (defined in the next section), and $C >0$ is a constant. 
Given $n$ policy updates and phases of length $m(s+s')$, the exploration term is bounded as $E_T = O(n m s')$, the pseudo-regret is bounded as $ \PR_T =O(m s\sqrt{n} + n \sqrt{m} s + \eps_0 T)$, and the terms $W_T$ and $V_T$ are of order $O(n m)$. By optimizing the terms, we obtain $n=T^{2/5}$, $m=T^{2/5}$, $s=T^{1/5}$, $s'=\log(T)$, and the corresponding regret is of order $\widetilde O(T^{4/5} + \eps_0 T)$.

%% file: new_offpolicy.tex
\section{Least-squares Monte-Carlo estimation with exploration}

Our approach is to directly solve a least-squares problem and find a solution $\widehat w\in\Real^d$ such that $Q_\pi \approx \Psi \widehat w$. In order to do so, we use the definition of the value of policy $\pi$, 
\begin{equation}
\label{eq:defn-val}
Q_\pi(x,a) = c(x,a) - \lambda_\pi + \E[V_\pi(x')|x,a]\,,\qquad   V_\pi = \sum_{t=1}^\infty P_\pi^t (c_\pi - \lambda_\pi) \;.
\end{equation}
Unlike TD methods, this approach does not rely on the Bellman equation.

Let $u$ be the uniform distribution over the action space and let $\nu = \mu_e\otimes u$ be the data-generating distribution.\footnote{We define $\mu_e\otimes u$ to be the distribution on $[S]\times [A]$ that puts the probability mass $\mu_e(x) u(a)$ on pair $(x,a)$.} Let $\widetilde w = \argmin_w \norm{Q_\pi - \Psi w}_{\nu}$ 
be the best linear value function estimator with respect to norm $\norm{.}_{\nu}$. The irreducible approximation error is the largest $\eps_0$ such that $\norm{Q_\pi - \Psi \widetilde w}_{\nu} \le \eps_0$ uniformly for all policies. We use $W_{\text{max}}$ to denote a constant such that $\norm{\widetilde w} \le W_{\text{max}}$ uniformly for all policies. 

Let $s$ be a sufficiently large integer (polynomial in the mixing time of policy $\pi$). Let $ \{(x_1,a_1,R_1), \dots, (x_m,a_m,R_m)\}$ be a dataset, where for each $j\in [m]$, $x_j$ is a state sampled under exploration distribution $\mu_e$, $a_j$ is sampled from the uniform distribution $u$, and $R_j$ is a trajectory generated under policy $\pi$ starting at state $x'_j$ where $x'_j$ is the first state observed after taking action $a_j$ in state $x_j$. The trajectory $R_j$ has the form $R_j=\{x'_j, a^{(j)}_0, c_0^{(j)}, x_1^{(j)},  \dots, x_{s-1}^{(j)}, a_{s-1}^{(j)}, c_{s-1}^{(j)}\}$, where $x_k^{(j)} \sim P_\pi^{k}(.|x'_j)$ is the state obtained by starting from state $x'_j$ and running policy $\pi$ for $k$ rounds, and $c_k^{(j)}=c(x_{k-1}^{(j)}, a_{k-1}^{(j)})$. 
\beq
\label{eq:estQ}
Q^*_\pi(x_j,a_j) = c(x_j,a_j) - \lambda^*_\pi +  V^*_\pi(x'_j) \,,\qquad V^*_\pi(x'_j) 
\eeq

Consider the dataset $\cD = \{(x_1,a_1,  Q^*_\pi(x_1,a_1)), \dots, (x_m, a_m, Q^*_\pi(x_m, a_m))\}$. Let $\widehat w$ be the solution of the following least-squares problem
\begin{align*}
\widehat w &= \argmin_w \sum_{j=1}^m (Q^*_\pi(x_j,a_j) - w^\top \psi(x_j,a_j))^2 
\end{align*}
The state-action value estimate is $\widehat Q_\pi = \Psi \widehat w$.

We will analyze the above estimation procedure in the next section. Although we treat each trajectory $(x_j, a_j, R_j)$ as a single sample, in practice it is often beneficial to use all data. For example, first-visit Monte-Carlo methods rely on returns from the first time a state-action pair is encountered in a trajectory, while every-visit Monte Carlo methods average returns obtained from all observations of a state-action pair. We will refer to the analyzed approach as one-visit. 
 Also, the choice of estimator $ \lambda^*_\pi = c_s^{(j)}$ is to simplify the analysis. The estimate $\lambda^*_\pi = \frac{1}{s}\sum_{k=1}^s c_k^{(j)}$ might be more appropriate in practice. 

\subsection{Value estimation}
\label{sec:val-est}

First, we show that the bias due to the finite rollout length $s=T^{1/5}$ is exponentially small. 
\begin{lemma}
\label{lemma:bias}
For the choice of $s=T^{1/5}$ and for $T \ge (2 \kappa \log(4 T^{1/5}))^5$, we have that
\[
\abs{\lambda_\pi - \E[\lambda^*_\pi]} \le 2 e^{-T^{1/5}/\kappa}\,, \qquad \abs{\E[V^*_\pi(x'_j)] - V_\pi(x'_j)} \le e^{-T^{1/5}/(2\kappa)} \;.
\]
\end{lemma}
The proof is given in Appendix~\ref{app:bias}. Given that these errors are exponentially small, we will ignore them to improve the readability.  

Let $\widehat D$ be the empirical estimate of $D:=D_{\mu_e\otimes u}$ using data $\cD$. 
The next lemma bounds the least-squares Monte-Carlo estimation error.

\begin{lemma}
\label{lem:MCerror}
Fix $\delta\in (0,1)$. Under the assumption that $\norm{\Psi^\top D \Psi} \geq \sigma$ , with probability at least $1-\delta$,  we have that
\[
\norm{\Psi( \widehat w - \widetilde w) }_{\mu_e \otimes u}  = O\left(\sigma^{-1/2} \sqrt{\log(1/\delta)/ m} \right) \,.
\]
\end{lemma}
The proof is given in Appendix~\ref{app:bias}. The above result implies that $\norm{\widehat w - \widetilde w} = O\left(\sigma^{-1} \sqrt{\log(1/\delta)/ m} \right) $. Given that $\norm{\widetilde w} \le W_{\text{max}}$, we conclude that for sufficiently large $m$, $\norm{\widehat w} \le 2 W_{\text{max}}$. Thus, the quantity $Q_{\text{max}}$ in \cref{eq:poliprgen} can be chosen to be $Q_{\text{max}} = 2 W_{\text{max}} \max_{x,a} \norm{\psi(x,a)}$.

\subsection{Comparison with existing work}

Our value estimation approach is related to off-policy temporal difference methods such as LSTD and LSPE in the sense that those methods attempt to solve the projected Bellman equation $\Psi w = \Pi_{\pi_e} (c - \lambda_\pi + H_\pi \Psi w)$. 
where the projection matrix is weighted by the distribution of $\pi_e$, and the transition matrix $H_\pi$ corresponds to the target policy. The goal is again to bound the estimation error in the $\nu_{\pi_e}$-weighted norm. However, while some analysis for off-policy LSTD exists for discounted MDPs \citep{Yu-2010},  we are not aware of any similar results for average-cost problems. In fact, it is known that off-policy LSPE can diverge due to the matrix $\Pi_{\pi_e} H_\pi$ not being contractive in the $\nu_{\pi_e}$-weighted norm \citep{bertsekas2011approximate}. In comparison to off-policy Monte-Carlo methods, our approach benefits from the fact that returns are estimated using on-policy rollouts, and hence we do not require importance weights.

Our bound has two advantages compared with the available results for on-policy LSTD and LSPE. First, the available bounds for these methods involve certain undesirable terms that do not appear in our result. For example, \citet{politex} show an error bound of the form $\norm{\Psi (\widehat w -  w^{\textsc{TD}})}_{\nu_\pi} \le O\left(\frac{1}{(1-\alpha) \sqrt{m}} \right)$. Here, $w^{\textsc{TD}}$ is the TD fixed point, $\widehat w$ is the finite sample LSPE estimate for policy $\pi$, and $\alpha\in [0,1]$ is the contraction coefficient of mapping $\Pi_\pi H_\pi$ with respect to norm $\norm{.}_{\nu_\pi}$, i.e. $\norm{\Pi_\pi H_\pi}_{\nu_\pi} \le \alpha$. \citet{Tsitsiklis-VanRoy-1999} show that $\alpha$ is smaller than one. However, it can be arbitrarily close to one, which will make the error bounds meaningless. Similar quantities appear in the error bounds for the LSTD algorithm. In contrast, our bounds do not depend on the measure of contractiveness. Second, the TD fixed point solution is not the same as the best possible estimation in the linear span of the features, and this introduces an additional approximation error. Let  $\rho =  \norm{(I - \Pi_\pi H_\pi)^{-1} \Pi_\pi H_\pi (I - \Pi_\pi)}_{\nu_\pi}^2$.
Theorem~2.2 of  \citet{yu2010error} shows that we might lose a multiplicative factor of $\sqrt{1+\rho}$: 
\beq
\label{eq:approx_yb}
\norm{Q_\pi -  \Psi w^{\textsc{TD}}}_{\nu_\pi} \leq \sqrt{1 + \rho} \min_{w} \norm{Q_\pi - \Psi w}_{\nu_\pi} \,.
\eeq
In contrast, we aim to get a better estimate by minimizing the error directly without imposing these constraints.  

%% file: experiments.tex
\section{Experiments}

 \begin{figure}[!t]
 \centering
\includegraphics[width=0.9\linewidth, trim=3cm 0.15cm 4cm 0.5cm, clip]{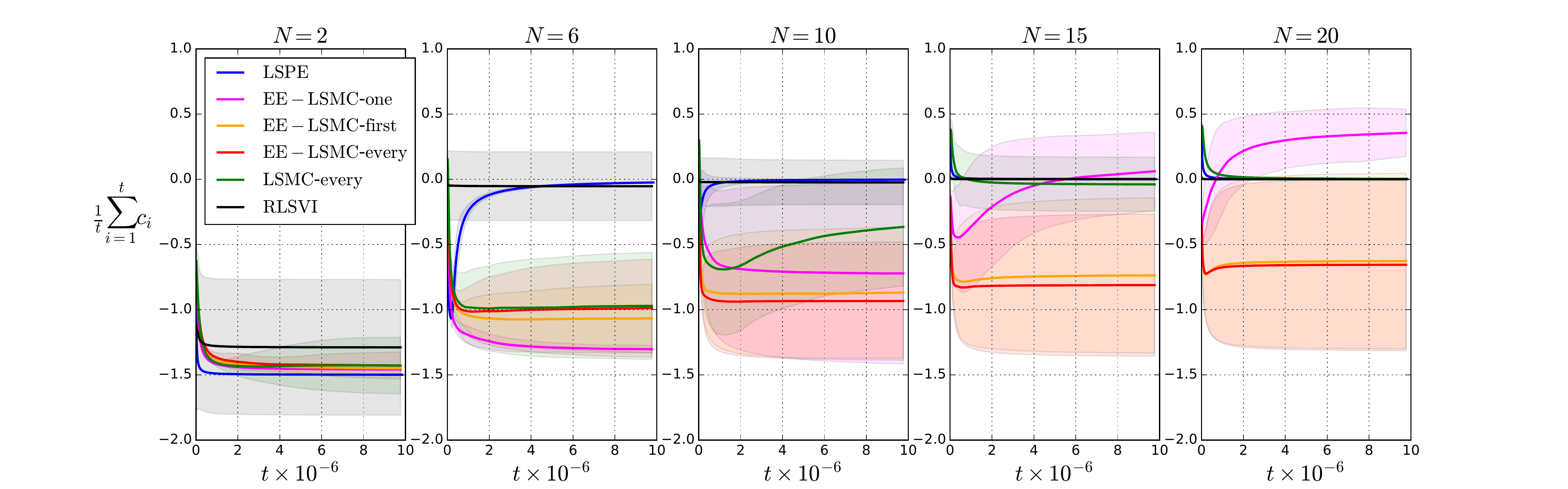}
\caption{Average cost of \alg with LSPE and LSMC, \eealg with LSMC, and RLSVI on the DeepSea environment.}
\label{fig:deepsea}
\end{figure}
\subsection{DeepSea environment}

We first demonstrate the advantages of our exploration scheme on a small-scale environment known as DeepSea (see also \citet{Osband-Wen-VanRoy-2016}), in which exploration can nonetheless be difficult. 
The states of this environment comprise an $N \times N$ grid, and there are two actions. The environment transitions and costs are deterministic. On action 0 in cell $(i, j)$, the environment transitions down and left, to cell $((i + 1) \mod N, \max(0, j-1))$. On action 1 in cell $(i, j)$, the environment transitions down and right, to cell $((i + 1) \mod N, \min(N-1, j+1))$.  The agent starts in state $(0, 0)$. The reward (negative cost) in the bottom-right state $(N-1, N-1)$ is $2N$ for any action. For all other states, the reward for action 0 is 0, and the reward for action 1 is -1. Thus, during the first $N$ steps (and in the episodic version of this task), the agent can obtain a positive return only if it always takes action 1, even though it is expensive.  In the infinite horizon setting, an optimal strategy first gets to the right edge and then takes an equal number of left and right actions, and has an average reward close to $1.5$. A simple strategy that always takes action 1 obtains an average reward close to $1$.
We represent states as length-$2N$ vectors containing one-hot indicators for each grid coordinate. 
 
We experiment with \alg-LSPE, \eealg-LSMC, and \alg-LSMC, i.e. \alg with value estimation using LSMC and no exploration. We also evaluate an online version of RLSVI \citep{Osband-Wen-VanRoy-2016} with linear function approximation, similar to the version described in \citet{politex}. 
For exploration, we use a policy that always takes action 1 and runs for $s' = N/2$ steps in each rollout. This policy can help discover the hidden reward but incurs additional costs.
 For LSMC, we evaluate first-visit, every-visit, and one-visit (i.e. just using the first sample) return estimates computed from length-$s$ rollouts. 
The results are shown in Figure~\ref{fig:deepsea} as costs, i.e. negative rewards. On a small $2\times2$ grid, all policies achieve the lowest cost. However, as the grid size increases, RLSVI and no-exploration \alg converge to the suboptimal policy which always takes action 0.  The performance of one-visit LSMC with exploration also deteriorates for higher $N$, and costs are positive due to exploration segments, suggesting that longer runs (more samples) are required in this case.

\subsection{Sparse Cartpole with function approximation}

In the next experiment, we examine whether the promising theoretical results presented in this paper lead to a practical algorithm when applied in the context of neural networks. We take the classic Cartpole (aka.\ Inverted Pendulum) problem \citep{tassa2018deepmind} (\cref{fig:cartpole} right), where the objective is to balance up a pole attached to a cart by only moving the cart left and right on the $x$ axis. The observation is a tuple consisting of the $x$ coordinate of the cart and its velocity, the cosine and sine of the angle $\theta$ of the pole compared to the upright position, and the rate of change of this $\theta$. There are three actions which correspond to applying force to the cart towards the left or right or applying no force. Each episode begins with the pole hanging downwards and ends after 1000 timesteps. There is a small cost of -0.1 for any movement of the pole. A reward of 1 is collected whenever the pole is almost upright and the cart is centered (with a controllable threshold). This is a difficult exploration problem as the rewards are sparse; in particular, no reward is seen at intermediate states which the agent has to nevertheless explore to solve the problem. 
We compare \eealg to \alg, where for \eealg we make use of a separate exploration policy that was trained with a separate reward function that encourages the exploration of states where the pole is swung up. We approximate state-action value functions using neural networks. We execute policies that are based only on the most recent $n$ neural networks, where $n$ is a parameter to be chosen. Further implementation details are given in \cref{app:cartpole}. Rather than evaluating learned policies after training, in line with the setting of this paper, we plot the total reward obtained by the agent against the total number of time steps.
From \cref{fig:cartpole}, we can see that \emph{without} the exploration policy, \alg never finds the solution to the problem, and learns to not move the cart at all. On the other hand, \eealg takes advantage of the exploration policy and manages to learn to collect rewards.
In contrast, we also ran experiments on Ms Pacman (Atari), where we found that mixing in an exploration policy did not help (details in \cref{app:pacman}). 

\begin{figure}[!t]
\begin{minipage}{0.50\textwidth}
\centering
\includegraphics[width=1.0\linewidth]{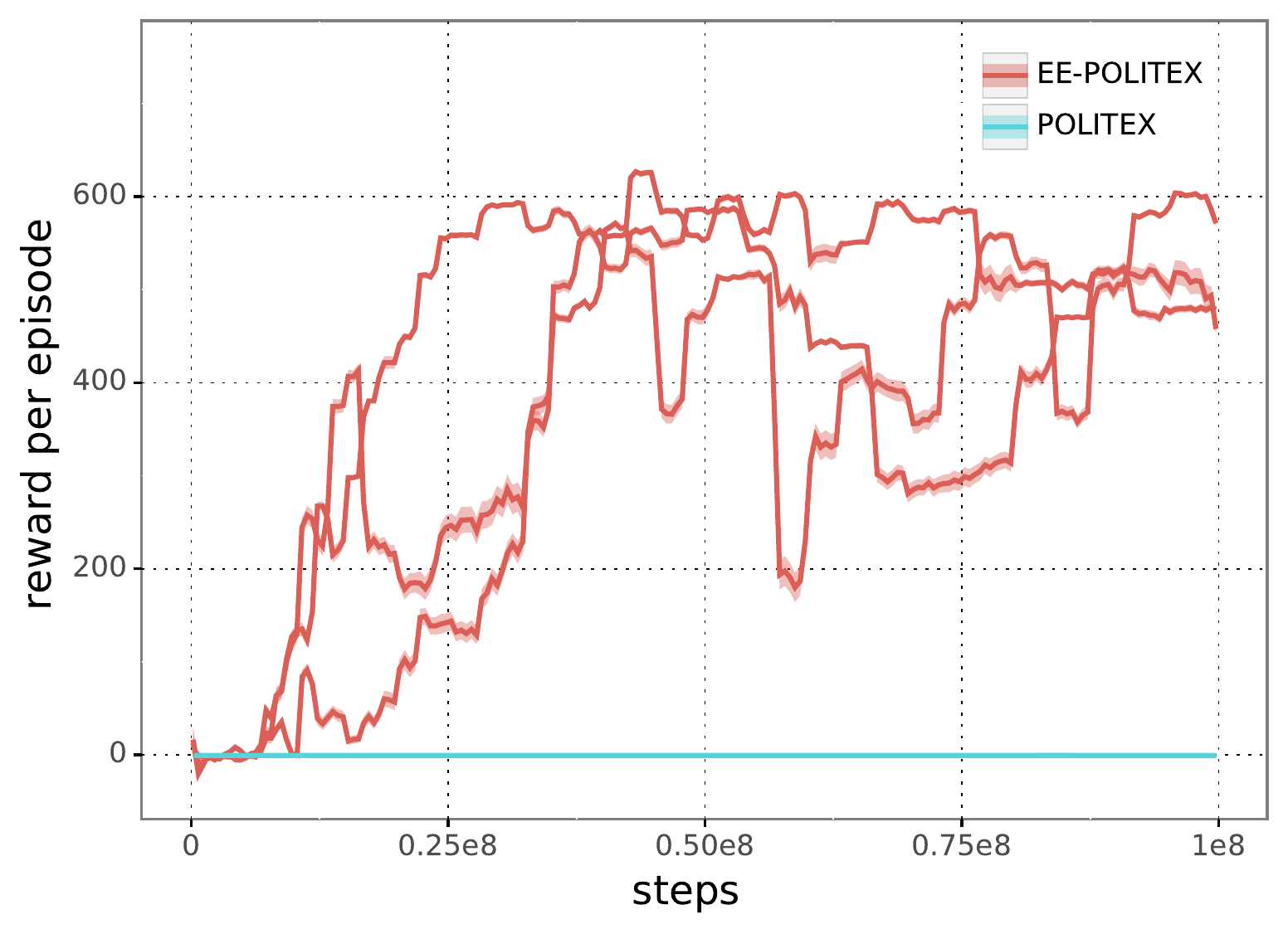}
\end{minipage}
\begin{minipage}{0.03\textwidth}
~
\end{minipage}
\begin{minipage}{0.45\textwidth}
\centering
\includegraphics[width=1.0\linewidth]{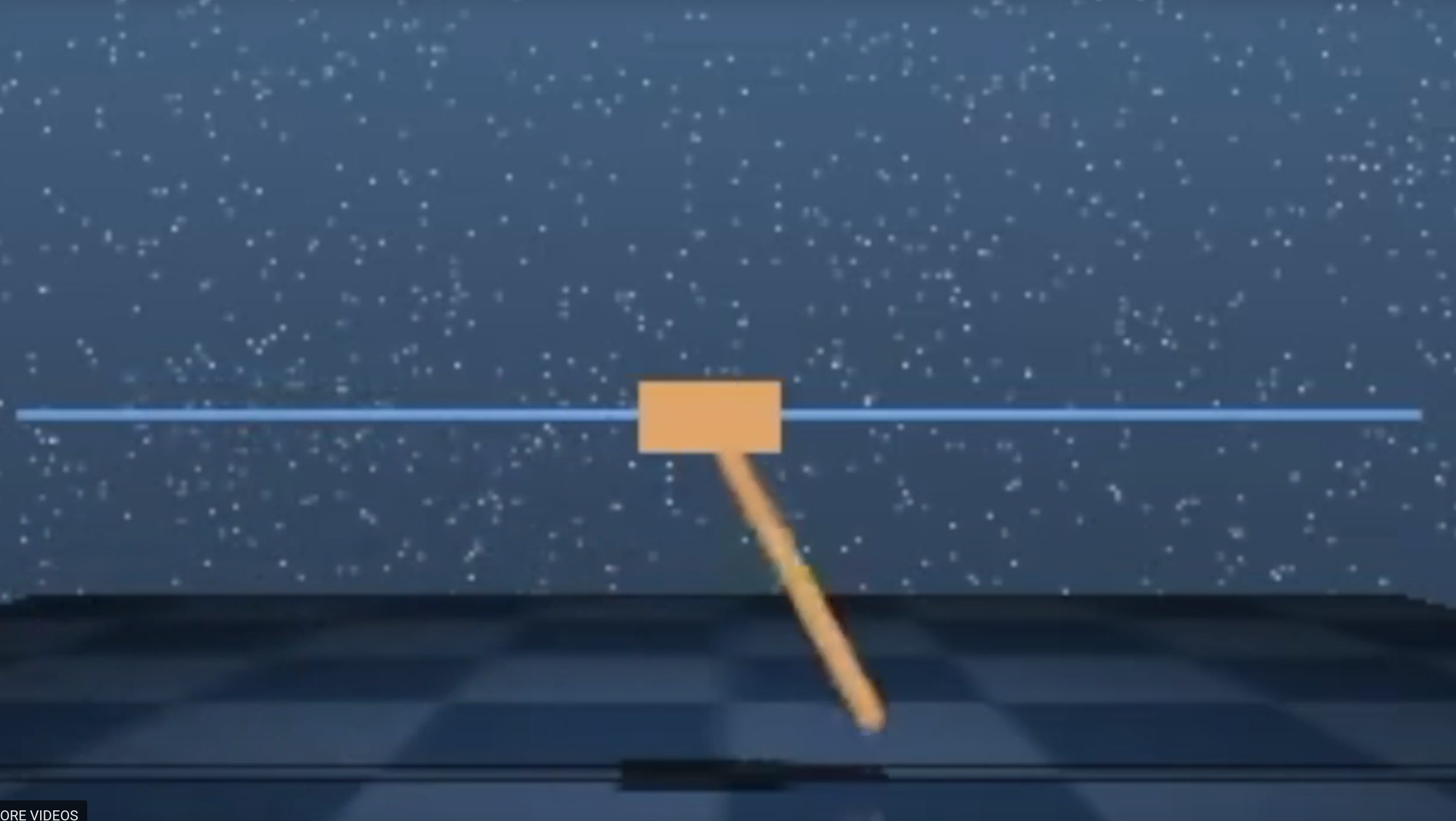}
\end{minipage}
\vspace{-0.2cm}
\caption{Left: Rewards obtained by the agents during the course of an episode. The plots are based on three runs of each algorithm with different random seeds. In the case of \eealg, rewards are only plotted for episodes where the agent did not follow the exploration policy. Right: A visualisation of an example state from the Cartpole environment.}
		\label{fig:cartpole}
\end{figure}

%% file: discussion.tex
\section{Discussion}

We have proposed an exploration strategy that utilizes a fast-mixing exploratory policy. This strategy  can be used with an action-value estimation algorithm that learns from on-policy trajectories whose initial states are chosen from the stationary state-distribution of the exploration policy. One such algorithm is least-squares Monte Carlo, for which we have provided an analysis of the estimation error.  
 Integrating our exploration scheme into the \alg algorithm  results in a new algorithm that enjoys sublinear regret guarantees under much weaker assumptions than previous work. The exploration scheme was demonstrated to improve empirical performance in difficult environments. 
While much work has been devoted to learning exploration policies, an interesting open problem for future work is learning non-trivial exploration policies which span the state-action feature space (and do not depend on returns).

%% file: appendix.tex
\section{Proof of Lemma~\ref{lem:bound-VW}}
\label{sec:vwproof}

We prove that if exploration-enhanced \alg has a total of $2nm$ policy restarts, and each policy is executed for at most $s$ timesteps, then with probability at least $1-\delta$, we have that
\begin{align*}
 |W_T| &\leq \kappa + 4\kappa\sqrt{2T\log\left(2/{\delta}\right)}\,,\\
|V_T| &\leq nm\kappa + 4 \kappa n\sqrt{4 ms \log\left(2T/{\delta}\right)}.
\end{align*}

\begin{proof}
For $W_T$, the proof is the same as in Lemma~4.4 of \citet{politex}, and we provide only a sketch. 
We decompose the term $W_T$ as follows:
  \begin{equation}\label{eq:gamma_decomp}
    W_T 
     = \sum_{t = 1}^T \lambda_{\pi^*} - c_t^*
     = \sum_{t = 1}^T \lambda_{\pi^*} -\E[c_t^*] + \sum_{t=1}^T \E[c_t^*] -c_t^* 
     = \underbrace{\sum_{t = 1}^T (\nu^* - \nu_t^*)^\top c}_{W_{T,1}} 
     + \underbrace{\sum_{t=1}^T \E[c_t^*] -c_t^*}_{W_{T,2}},
  \end{equation}
  where 
  $\nu^*$ denotes the stationary distribution of the policy $\pi^*$ 
  and $\nu_t^*$ is the state-action distribution after $t$ time steps
	and we used that $\E[c_t^*] = \nu_t^\top c$.  
We bound  $W_{T,1}$ in the equation above using the uniform mixing assumption:
  \begin{equation}
  \label{eq:gamma_bound1}
    \left\vert\sum_{t = 1}^T (\nu^* - \nu_t^*)^\top c\right\vert 
    \leq \sum_{t=1}^T \|\nu^* - \nu_t^* \|_1\|c\|_\infty 
    \leq \sum_{t=1}^T \exp(-t/\kappa) \leq \kappa,
  \end{equation}
where we have assumed that $\norm{c}_\infty \leq 1$. The second term can be written as $W_{T, 2} = \ip{B_0-B_T,c}$, where $B_i = \E[\sum_{t=1}^T X_t | X_1, ..., X_i]$, and $X_t$ is a binary indicator vector with a non-zero at the linear index of the state-action pair $(x_t, a_t)$. The bound $W_{T, 2} \leq 4\kappa\sqrt{2T\log\left(2/{\delta}\right)}$ follows from the fact that $(B_i)_{i=1}^T$ is a vector-valued martingale with a bounded difference sequence and the Azuma inequality.

The bound on $V_T$ follows similarly by noticing that \alg makes $2nm$ policy switches of length $s$.  Decomposing $V_T = V_{T, 1} + V_{T, 2}$ analogously to $W_{T}$, we have that $V_{T, 1} \leq nm \kappa$. For $V_{T,2}$, we have that with probability at least $1-\delta$, each length-$s$ segment $\sum_{i=t}^{t+s-1}\E[c_t] - c_t$ is bounded by $4\kappa \sqrt{4s \log (2/\delta)}$.  Within each of $n$ iterations, we have $m$ such identically-distributed bounded segments corresponding to the same policy.  A similar observation applies to length-$s'$ segments corresponding to the exploration policy.
Thus, using a union bound and Hoeffding's inequality, we have 
\begin{equation}\label{eq:alpha_bound_final}
|V_T| \leq 2 n m \kappa + 4 n \kappa\sqrt{4 m s \log(2m(s + s') /\delta)}\, .
\end{equation}

\end{proof}

\section{Proofs of Section~\ref{sec:val-est}}
\label{app:bias}

First, we prove that for the choice of $s=T^{1/5}$ and for $T \ge (2 \kappa \log(4 T^{1/5}))^5$, 
\[
\abs{\lambda_\pi - \E[\lambda^*_\pi]} \le 2 e^{-T^{1/5}/\kappa}\,, \qquad \abs{\E[V^*_\pi(x'_j)] - V_\pi(x'_j)} \le e^{-T^{1/5}/(2\kappa)} \;.
\]
\begin{proof}[Proof of Lemma~\ref{lemma:bias}]
Let $ \one{x}$ be a binary indicator vector corresponding to a state $x$. 
By \cref{eq:estQ},
\begin{align*}
\E[V^*_\pi(x'_j)] &=  \sum_{k=1}^s \E[ c_k^{(j)} -  \lambda^*_\pi ] \\ 
&=  \one{x'_j}^\top \sum_{k=1}^s  P_\pi^k (c_\pi - \lambda_\pi) + \sum_{k=1}^s \E[ \lambda_\pi -  \lambda^*_\pi ] \\
&= V_\pi(x'_j) - \one{x'_j}^\top \sum_{k=s+1}^\infty  P_\pi^k (c_\pi - \lambda_\pi) + \sum_{k=1}^s \E[ \lambda_\pi -  \lambda^*_\pi ] \\
&= V_\pi(x'_j) - \one{x'_j}^\top (P_\pi^s - {\bf 1} \mu_\pi^\top ) V_\pi + \sum_{k=1}^s \E[ \lambda_\pi -  \lambda^*_\pi ]\,, \qquad \text{By \cref{eq:defn-val} and $\mu_\pi^\top V_\pi = 0$.}
\end{align*}
By the uniformly fast mixing assumption, 
\[
\abs{\lambda_\pi - \E[ \lambda^*_\pi]} = \abs{\mu_\pi^\top c_\pi - \one{x'_j}^\top P_\pi^s c_\pi} \le 2 e^{-s/\kappa}\,,
\]
and
\[
\abs{\E[V^*_\pi(x'_j)] - V_\pi(x'_j)} \le 2 (1 + s) e^{-s/\kappa} \le e^{-T^{1/5}/(2\kappa)} \,,
\]
where the last inequality follows by the choice of $s=T^{1/5}$ and for $T \ge (2 \kappa \log(4 T^{1/5}))^5$.
\end{proof} 

\begin{proof}[Proof of Lemma~\ref{lem:MCerror}]
We have that $\Psi^\top \widehat D \Psi \widehat w = \Psi^\top \widehat D  Q^*_\pi$ and $\Psi^\top D \Psi \widetilde w = \Psi^\top D Q_\pi$. Thus,
\begin{align*}
\Psi^\top D \Psi (\widetilde w - \widehat w) + \Psi^\top (D - \widehat D) \Psi \widehat w = \Psi^\top (D - \widehat D) Q_\pi + \Psi^\top \widehat D (Q_\pi - Q^*_\pi) \;.
\end{align*}
Using the assumption that $\norm{\Psi^\top D \Psi} \geq \sigma$ for the exploration policy,
\begin{align}
\label{eq:errTheta}
 \norm{D^{1/2} \Psi( \widehat w - \widetilde w)} & = \norm{\Psi( \widehat w - \widetilde w) }_{\mu_e \otimes u}  \notag \\ 
 & \le \frac{1}{\sqrt{\sigma}} \norm{\Psi^\top \widehat D ( Q_\pi - Q^*_\pi)} + \frac{1}{\sqrt{\sigma}} \norm{\Psi^\top ( D - \widehat D) (Q_\pi - \Psi \widehat w)}   \;.
\end{align}
For the second term, consider that for any vector $z\in\Real^S$,
\[
\Psi^\top D z = \sum_{x,a} \mu_e(x) u(a) z(x) \psi(x,a)\,, 
\]
and $\Psi^\top \widehat D z$ is simply its estimate using $m$ i.i.d. samples. So the second term can be bounded as $O(1/\sqrt{m})$. For the first term, let $v = \widehat D ( Q^*_\pi - Q_\pi)$, and notice that only $m$ elements of this vector are non-zero and these elements are random variables with zero expectation. So for any deterministic vector $z\in [0,1]^S$, by Hoeffding's inequality and with probability at least $1-\delta$,
\[
z^\top v = \sum_{i=1}^m v(i) z(i) \le \frac{1}{m}\sqrt{\log(1/\delta)\sum_{i=1}^m z_i^2} \le \sqrt{\frac{\log(1/\delta)}{m}} \;.
\]
\end{proof}

\section{Neural network experiment setup}

\subsection{Cartpole experiment}

\label{app:cartpole}
Our implementation of the Cartpole experiment is based on \citet{horgan2018distributed}, which is a distributed implementation of DQN \citet{mnih2015human}, featuring 
Dueling networks \citep{DBLP:journals/corr/WangFL15},
N-step returns \citep{DBLP:journals/corr/AsisHHS17},
Prioritized replay \citep{schaul2015prioritized}, and
Double Q-learning \citep{DBLP:journals/corr/HasseltGS15}. 
To adapt this to \alg, we used TD-learning and Boltzmann exploration with the learning rate $\eta$ set according to SOLO FTRL by \citep{OrDa15}: For a given state $x$,
\[
\eta = \alpha \sqrt{2.75}\sqrt{\frac{\log d}{P_t}},
\]
where $\alpha=10$ is a tuneable constant multiplier (chosen based on preliminary experiments); 
$d$ is the number of actions in the game and
\[
P_t = \min_{c\in\R} \sum_{i=1}^t \left\| Q_i(x, \cdot)- c\mathds{1}\right\|^2_\infty,
\]
where $Q_i(x,\cdot)$ are the state-action values for all past $Q$-networks indexed from $1$ to the current timestep $t$, $\mathds{1}$ is a vector of all ones and the minimisation over $c$ achieves robustness against the changing ranges of state-action values. The minimisation is one-dimensional convex optimisation problem which we solve numerically.

Both methods used the same neural network architecture of a dueling MLP Q-network with one hidden layer of size 512. Each learner step uniformly samples a batch of 128 experiences from the experience replay memory. The actors take 16 game steps in total for each learning step. We enter a new phase and terminate the current phase when the number of learning steps taken in the current phase reaches 100 times the square root of the total learning steps taken. 
When a new phase is started, the freshly learned neural network is copied in a circular buffer of size 10, which is used by the actors to calculate the averaged Q-values, weighted by the length of each phase. For \eealg, we split each phase into ``mini-phases'' of length scaling with the square root of the current phase lenght. Each mini-phase consists of following the exploration policy for 1000 steps, and following the learned policy for a number of steps that scales with the square root of the current phase length. This corresponds to the settings $s'=1000$, $n=T^{1/2}$, $m=T^{1/4}$, $s=T^{1/4}-s'$.

\subsection{Ms Pacman (Atari) experiment}

\label{app:pacman}
We also compared \alg with \eealg on the Atari game Ms Pacman. Here, the exploration strategy mixed into \eealg was trained such that the agent had to drive Ms Pacman to random positions on the map. This policy collects significantly less rewards than one trained to maximise rewards, but explores the map well. We found that this exploration didn't help \eealg; in particular, it seems that exploration is relatively simple in this game, and enough exploration is performed even without mixing in the exploration policy. The results are shown in \cref{fig:pacman}; the exploration only hurts \eealg as the agent collects less reward from exploration episodes.

\begin{figure}[!t]
\centering
\includegraphics[width=0.55\linewidth]{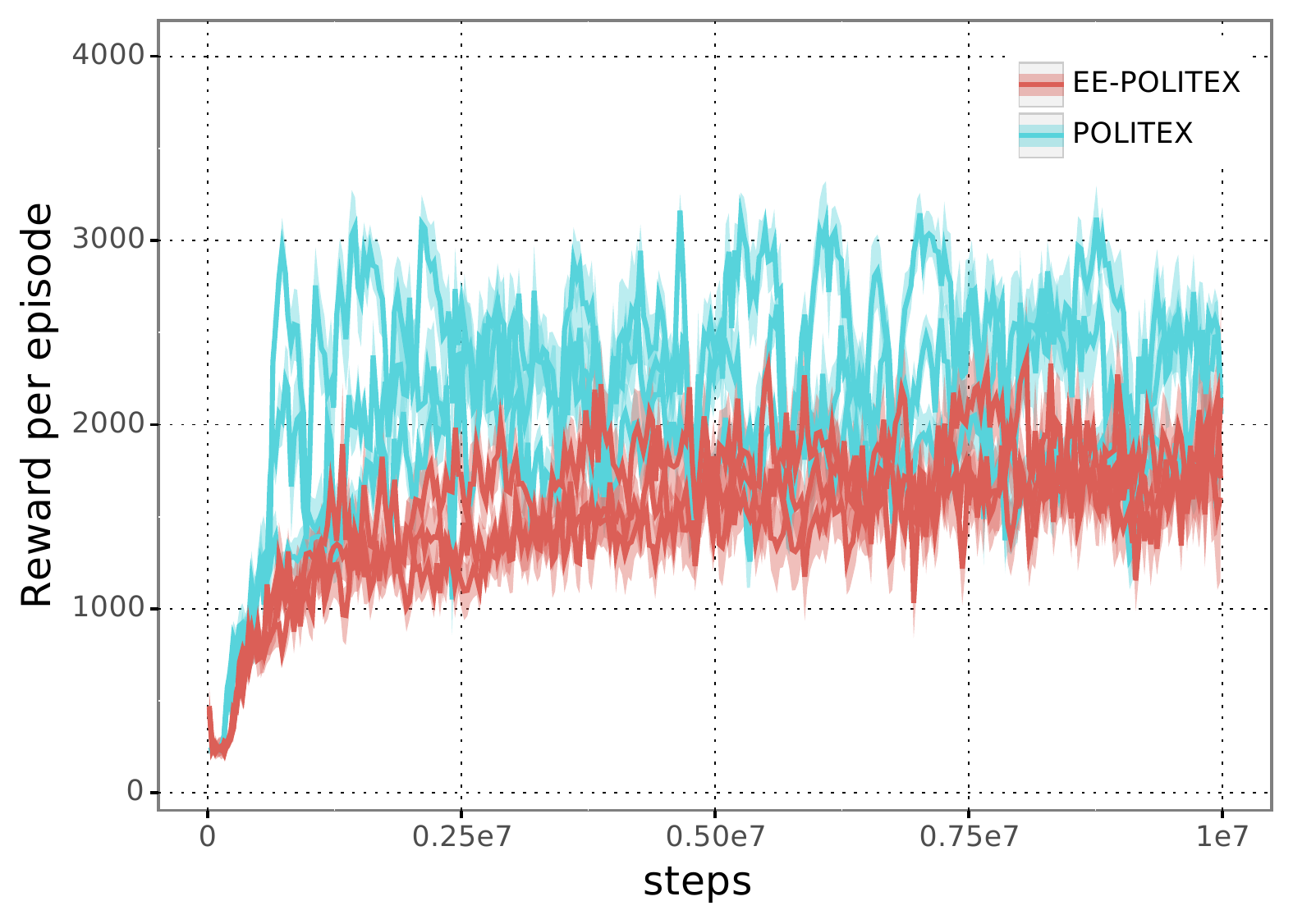}
\vspace{-0.2cm}
\caption{Rewards obtained by the agents during the course of an episode. The plots are based on three runs of each algorithm with different random seeds.}
\label{fig:pacman}
\end{figure}